\documentclass[letterpaper]{article}
\usepackage{uai2019}
\usepackage[margin=1in]{geometry}

\usepackage{times}

\usepackage{bbm}
\usepackage{amsmath}
\usepackage{amsthm}
\usepackage{graphicx}
\usepackage{amssymb}
\usepackage{mathrsfs}
\usepackage{centernot}
\usepackage{algpseudocode}
\usepackage{algorithm}
\usepackage{tikz}
\usetikzlibrary{automata, positioning}
\newtheorem{theorem}{Theorem}
\newtheorem{proposition}[theorem]{Proposition}
\newtheorem{corollary}[theorem]{Corollary}
\newtheorem{lemma}[theorem]{Lemma}
\theoremstyle{definition}

\title{Issues Concerning the Realizability of Blackwell Optimal Policies in Reinforcement Learning}

\author{Nicholas Denis \\
nick.denis.1983@gmail.com\\
} 

\begin{document}

\maketitle

\begin{abstract}
$N$-discount optimality was introduced as a hierarchical form of policy- and value-function optimality, with Blackwell optimality lying at the top level of the hierarchy [17,3]. We formalize notions of myopic discount factors, value functions and policies in terms of Blackwell optimality in MDPs, and we provide a novel concept of regret, called Blackwell regret, which measures the regret compared to a Blackwell optimal policy. Our main analysis focuses on long horizon MDPs with sparse rewards. We show that selecting the discount factor under which zero Blackwell regret can be achieved becomes arbitrarily hard. Moreover, even with oracle knowledge of such a discount factor that can realize a Blackwell regret-free value function, an $\epsilon$-Blackwell optimal value function may not even be gain optimal. Difficulties associated with this class of problems is discussed, and the notion of a policy gap is defined as the difference in expected return between a given policy and any other policy that differs at that state; we prove certain properties related to this gap. Finally, we provide experimental results that further support our theoretical results.\footnote[1]{work in progress.}
\end{abstract}

\section{Introduction}
\paragraph{}When is one policy better than another, and how does one arrive at the best policy? Additionally, is there a difference between the theoretical answers to these questions and how they are addressed in practice? Within the reinforcement learning and Markov decision process community, these questions are fundamental and nothing new. Indeed, though these questions have been well defined and well studied, this paper reconsiders important issues with solutions to MDPs and RL problems. Specifically, we explore the role of the discount factor $\gamma$ in finding an optimal policy, $\pi_{\gamma}^{*}$, and value function $V^{\pi^{*}}_{\gamma}$. Once $\gamma$ is chosen, though an (approximately) optimal solution may be returned by some algorithmic solution, it may still be unsatisfactory in some regards (as demonstrated by OpenAI with the Coastrunners domain). In this paper we explore the relationship between $\gamma, \epsilon$ in arriving at an $\epsilon$-optimal policy, as well as a researchers preference or evaluation of such a policy. We discuss issues surrounding selecting $\gamma$ and $\epsilon$ without any domain knowledge of the problem, and how even theoretically sound algorithms such as PAC-MDP solution methods can produce policies that, though satisfy being $(\epsilon, \delta)$-PAC, are still not even gain optimal. Especially difficult are long-horizon problems (LHPs) with sparse rewards. Motived by such problems we introduce a novel concept of regret, called \textit{Blackwell Regret}, $\mathcal{R}_{B}$, which compares the expected return of a given policy to that of a Blackwell optimal policy, evaluated at an appropriate value of $\gamma \in [0,1)$. We believe Blackwell regret is more akin to how humans experience regret when comparing oneself to the highest of standards. We formalize the notion of myopic discount factors and policies and introduce a notion of $\gamma$ being Blackwell realizable. We discuss how policies that minimize Blackwell regret are fundamentally difficult to solve for, as recent literature has hinted at for long horizon problems (LHP's)[10]. This is due to the existence of \textit{pivot states} where discovering the Blackwell optimal policy hinges on discerning the values of a Blackwell optimal policy and a non-Blackwell optimal policy which can be arbitrarily close at a given state. Even with oracle knowledge of $\gamma^{*}$, the \textit{infimum} $\gamma$ that can induce a Blackwell optimal and Blackwell regret-free policy, $\forall \epsilon > 0$, an $\epsilon$-accurate Blackwell optimal policy may not be Blackwell optimal, and in fact may not even be gain optimal. We provide experimental results using PAC-MDP algorithms that demonstrate this phenomenon. Motivated by these findings, we argue the need for progress within three areas of theoretical research: 1) Analytical solution methods for Blackwell optimal policies; 2) provable convergent algorithms for solving n-discount optimal policies; 3) goal based and human preference based RL. Our focus is on the latter.

\section{Background}
\subsection{Markov Decision Process}
\paragraph{}Recall that an MDP, $\mathcal{M}$, is an n-tuple $\langle \mathcal{S}, \mathcal{A}, p, R, \gamma \rangle$, where $\mathcal{S}$ is a finite state space, $\mathcal{A}$ is a finite action space, $p=p(s'\vert s,a)$ is the transition kernel, $R: \mathcal{S}\times \mathcal{A} \times \mathcal{S} \to [0, R_{max}]$ is the reward function and $\gamma \in [0,1]$ is the discounted current value associated to one unit of reward to be received one unit of time into the future. This work focusses on deterministic Markovian policies.

\subsection{Notation}
\paragraph{}
This work considers how $\gamma$ plays a role both in learning a policy as well as how it is used in evaluating the value function associated to a policy, perhaps learned with a different discout factor. For this reason, it is important to clearly separate $\gamma$ used to learn a policy $\pi_{\gamma}$, and $\gamma^{'} \neq \gamma$ used to evaluate that policy. Hence, by $\pi_{\gamma}$ we refer to a policy learned using $\gamma$, whereas $V^{\pi_{\gamma_{1}}}_{\gamma_{2}}(s)$ refers to the value of a state, when following policy $\pi$ $\textit{learned}$ using $\gamma_{1}$ (as defined just previously), however the value function is computed using $\gamma_{2}$. Symbolically, $V^{\pi_{\gamma_{1}}}_{\gamma_{2}}(s_{t}) = \mathbb{E}_{\pi_{\gamma_{1}}}\big\{\sum_{k=t}^{\infty} \gamma^{k-t}_{2}r_{k} \big\}$. 

\subsection{Optimality in MDP's}
\paragraph{}
If $\gamma = 1$ then we are considering undiscounted rewards, and for any infinite stream of rewards $G_{t} = \mathbb{E}\{r_{t} + r_{t+1} + r_{t+2} ..., \} = \mathbb{E}\{\sum_{k=t}^{\infty} r_{k} \}$. Since $G_{t}$ is often infinite, the \textit{gain} of a policy $\pi$ is defined, where 
\begin{align*}
\rho^{\pi} =  \lim_{T \to \infty} \frac{1}{T}\mathbb{E}_{\pi}\big\{\sum_{t=1}^{T}r_{t}\big\}.
\end{align*}
 Using the gain of a policy, an ordering, $\geq$, is defined on some policy class $\Pi$, where $\forall \pi_{1}, \pi_{2} \in \Pi$, $\pi_{1} \geq \pi_{2} \iff \rho^{\pi_{1}} \geq \rho^{\pi_{2}}$, with the strict inequality defined similarily.  It is worth noting that if we define \textbf{r}$^{\pi} = \{r_{1}, r_{2}, ... \}$ as the sequence of expected rewards from following policy $\pi$, then for any permutation $\sigma: \mathbb{N} \to \mathbb{N}$, any policy $\pi'$ whose sequence of expected rewards is $\sigma($\textbf{r}$^{\pi}) = \sigma\{r_{1}, r_{2}, ...\} = \{r_{\sigma(1)}, r_{\sigma(2)}, ... \} = \ $\textbf{r}$^{\pi'}$, then $\rho^{\pi} = \rho^{\pi'}$. Hence, the temporal ordering of rewards has no bearing on the value or gain of a policy when $\gamma = 1$. This is certainly not true for $\gamma < 1$.

\paragraph{}
Most commonly $\gamma \in [0,1)$. In this setting we can deal with infinite series of expected rewards as the partial sums converge geometrically fast in $\gamma$. The value of a state when considering discount factor $\gamma$, is  
\begin{align*}
V^{\pi_{\gamma}}_{\gamma}(s_{t}) = \mathbb{E}_{\pi_{\gamma}}\big\{\sum_{k=t}^{\infty}\gamma^{k-t}r_{k} \big\}.
\end{align*}
Since most frameworks assume rewards are bounded in some interval $[0, R_{max}]$, then $\forall \pi,$ $\forall s$, $V^{\pi}_{\gamma}(s) \leq V^{max}_{\gamma} = \frac{R_{max}}{1-\gamma}$. Such assumptions and the use of $V^{max}_{\gamma}$ are integral to theoretical bounds for algorithms and solution methods in RL and MDP's. Similar to the ordering on policies in the undiscounted setting, an ordering of policies $\pi_{\gamma}$ for fixed $\gamma \in [0,1)$ is used to order policies $\pi_{1}, \pi_{2} \in \Pi_{\gamma}$. Unlike the undiscounted setting, under $\gamma < 1$, two policies are not equivalent under permutation of the temporal sequence of rewards. Interestingly,  a value of $\gamma = 0$ is rarely used in the literature, and is often called \textit{myopic}. With $\gamma = 0$, the induced policy does not sufficiently account for the future {\it horizon} and in doing so is generally viewed to only lead to sub-optimal behaviour.

\paragraph{}
$\forall \gamma \in [0,1]$, we say a policy $\pi_{\gamma}^{*}$ is optimal if $V^{\pi_{\gamma}^{*}}_{\gamma}(s) \geq V^{\pi_{\gamma}}_{\gamma}$, $\forall \pi_{\gamma} \in \Pi_{\gamma}, \forall s \in \mathcal{S}$, where $V^{\pi_{\gamma = 1}^{*}}_{\gamma = 1} = \rho^{\pi_{\gamma}^{*}}$. Despite these notions of optimality being the most common in RL, there are other notions of optimality [13].

\subsubsection{Bias Optimality}
\paragraph{}
{Bias optimality was introduced to supplement the use of gain optimal policies when $\gamma = 1$. Since the gain of a policy only considers the asymptotic behaviour of a policy, two policies that have the same gain may experience different reward trajectories before arriving at the stationary distribution of the policy. For this reason the bias of a policy, defined as 
\begin{align*}
b^{\pi}(s) = \lim_{T \to \infty} \mathbb{E}\{\sum_{t=1}^{T}(r_{t}(s) - \rho^{\pi})\},
\end{align*}
and was introduced by [3]. For any finite state and action space MDP, a bias optimal policy always exists.

\subsubsection{$n$-discount Optimality}
\paragraph{}
$n$-discount optimality [17] introduces a hierarchical view of policy optimality in MDPs. A policy $\pi^{*}$ is $n$-discount optimal for $n \in \{-1,0,1,2,3,...\}$ if $\forall s \in \mathcal{S}$, and $\forall \pi \in \Pi$
\begin{align*}
\lim_{\gamma \to 1}(1-\gamma)^{-n}\big(V^{\pi^{*}}_{\gamma}(s) - V^{\pi}_{\gamma}(s) \big) \geq 0.
\end{align*}
It has been shown [17] that a policy is $-1-$discount optimal $\iff$ it is gain optimal, and a policy is 0-discount optimal $\iff$ it is bias optimal. Moreover, if a policy is $n$-discount optimal, then it is $m$-discount optimal $\forall m \in \{-1,0,...,n\}$. The strongest and most selective notion of optimality is that of $\pi^{*}$ being $n$-discount optimal $\forall n\geq -1$. Such a policy is referred to as being \textbf{$\infty$}-discount optimal.

\subsubsection{Blackwell Optimality}
\paragraph{}
A policy $\pi^{*}$ is Blackwell optimal if $\exists \gamma^{*} \in [0,1)$, such that
\begin{align*}
 V^{\pi^{*}}_{\gamma}(s) \geq V^{\pi}_{\gamma}(s), \ \forall \gamma \in [\gamma^{*}, 1), \ \forall \pi \in \Pi, \  \forall s \in \mathcal{S}.
\end{align*}
For finite state spaces such a $\gamma^{*}$ is attained [13]. Intuitively, a Blackwell optimal policy is one that, upon considering \textit{sufficiently far} into the future, as encoded as a planning horizon via $\gamma > \gamma^{*}$, no other policy has a higher expected cumulative reward.[17] showed that a policy is $\mathbf{\infty}$-discount optimal $\iff$ it is Blackwell optimal, hence Blackwell optimality implies all other forms of optimality, and for this reason is the focus of this work. Finally, [3] shows that for finite state and action space MDPs, there always exists a stationary and deterministic Blackwell optimal policy.

\section{Motivation for Blackwell Regret}

\paragraph{}Consider the infinite horizon MDP in Figure 1, with initial state $s_{0}$. Before proceeding, consider what you would do if you were in this MDP? What do you think is the best policy? What sort of solution would you hope that an RL algorithm return to you and how did you come to this conclusion?

\begin{figure}
\centering
\begin{tikzpicture}[scale=0.8, transform shape]
        \tikzset{node style/.style={state, 
                                    fill=gray!20!white,
                                    circle}}
            
         \node[node style]  (0) {$s_{0}$};
	\node[node style, right=of 0] (1) {$s_{1}$};
	\node[draw=none, right=of 1] (2) {$\cdots$};
	\node[node style, right=of 2] (H) {$s_{H}$};

    \draw[>=latex,
          auto=left,
          every loop]

	(0) edge[loop above] node{$\epsilon << 1$} (0)
	(0) edge[bend left, auto=left] node{0}(1)
	(1) edge[bend left, auto=left] node{0}(2)
	(1)edge[bend left, auto=right] node{0}(0)
	(2) edge[bend left, auto=left] node{0}(H)
	(2) edge[bend left, auto=right] node{0}(1)
	(H) edge[bend left, auto=right] node{0}(2)
         (H) edge[loop above] node{1} (H);
\end{tikzpicture}
\caption{\textbf{Distracting Long Horizon MDP Example.} For H $>> 1$, and initial state $s_{0}$.} \label{fig:M1}
\end{figure}
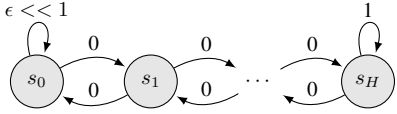

\paragraph{}
In wanting to maximize cumulative reward, it is hard to argue with any other action selection policy for the provided example than to always ``move right" towards the state $s_{H}$, and upon doing so, remain there. Why might someone consider any other policy? Why might a rational agent, with full oracle knowledge of the MDP, consider staying in $s_{0}$ to receive a reward of $\epsilon << 1$ at every time step for perpetuity? It is hard to account for why such a policy would be preferred over the policy that takes the agent to $s_{H}$, aside from laziness. Computationally, $V^{\pi_{stay}}_{\gamma} > V^{\pi_{right}}_{\gamma} \iff \epsilon > \gamma^{H}$. Hence, depending on $H$, $\epsilon, R_{max}$, the policy induced by $\gamma$ can be set appropriately in order to induce the desired policy behaviour.

\paragraph{}
Returning to $\gamma = 0$, it is widely accepted within the literature that $\pi_{\gamma}^{*}$ is myopic. We ask if it is possible for $\pi_{\gamma}^{*}$ to be myopic for $\gamma \neq 0$? Is $\gamma = 10^{-1000}$ myopic? $\gamma = 10^{-999}$? If we abstract what makes $\pi_{\gamma}^{*}$ myopic, it is the fact that $\gamma$ is not sufficiently large so as to provide the agent with the \textbf{possibility} of properly assessing the \textbf{optimal} value of states and actions, where this optimality is, in some sense, not defined with respect to $\gamma_{learn}$, the $\gamma$ used during learning, but rather with respect to some ideal policy or behaviour. Just as a child might seek to maximize immediate gratification (rewards) by eating candy before bed, which may be optimal given $\gamma = 0$, the role of a parent will be to convey the non-optimality of such a policy by noting that the yet to be experienced consequences (poor sleep, fussy behavior the following day), which can only be taken into consideration with $\gamma > 0$. This is paradoxical for the child, as they operate under $\pi^{*}_{\gamma = 0} = \pi_{eat candy}$, and hence $V^{\pi_{eatcandy}}_{\gamma = 0}$ is optimal from the perspective of $\gamma = 0$. The lesson the parent tries to impart to the child is to use $\gamma^{'} > \gamma = 0$ so that the child can learn $\pi^{*}_{\gamma^{'}}$. In this way, we intuitively compare $V^{\pi_{\gamma}}_{\gamma^{'}}$ to $V^{\pi_{\gamma^{'}}}_{\gamma^{'}}$. It is this intuition that we seek to formalize by noting that eating candy before bed does not sufficiently value the future, and for this reason we attempt to resist this myopic behaviour. In order to do so, sufficiently valuing the future means selecting a suitable $\gamma \in [0,1)$. We argue that this sufficiency is represented by the $\gamma^{*} \in [0,1)$ as found in the definition of a Blackwell optimal policy and value function.

We argue that the myopic behaviour, intuitively, is defined with respect to the strongest sense of optimality, Blackwell optimality. Note that $\forall \gamma \in [0,1)$, $\exists$ $\pi_{\gamma}^{*}$ [13]. So why, then, is $\pi_{\gamma}^{*}$ dismissed as myopic for $\gamma = 0$? It is still, after all, an optimal policy. We believe that this occurs since we intuitively understand that not all optimal policies are equal. It appears that all optimal policies are optimal, but some are more optimal than others. That is, though $\forall \gamma,$ $\pi_{\gamma}^{*}$ is optimal under $\gamma$, not all $\gamma$'s induce the policies or behaviours that a researcher prefers. This clearly highlights a common issue in machine learning, that of using a given objective function as a surrogate representation for \textit{what we want the algorithm to do}.

The hierarchical nature of policy optimality as expressed by $n$-discount optimality naturally captures this phenomenon, and we revisit this body of literature to help motivate why our sense of $\gamma$ being myopic has nothing to do with not being capable of finding $\pi_{\gamma}^{*}$, but rather, not finding $\pi_{\gamma^{*}}^{*}$, the $\gamma^{*}$ that characterizes a Blackwell optimal policy. We introduce a novel notion of regret, called \textit{Blackwell regret}, and relate the concept of a myopic $\gamma$ and policy to Blackwell regret. Our work looks at a simple class of MDPs, called \textit{distracting long horizon MDPs}, and show that even for such a simple class of environments, it is arbitrarily hard to select a $\gamma$ so as to arrive at a Blackwell optimal policy and value function that achieves zero Blackwell regret.

\section{Myopic $\gamma$, Blackwell Realizable and Blackwell Regret}
\paragraph{}
Looking at the MDP's in Figure 1, we intuitively get a sense of what the right policy is, and we agree that $\gamma = 0$ is myopic and will not produce the optimal policy. Moreover, we can check that $\forall \gamma \in [0, \sqrt[^H]{\epsilon})$ will suffer the same drawback. Since no formal definition of a myopic $\gamma$ can be found in the literature, we provide a definition.

\paragraph{Definition: Myopic $\gamma$ and Blackwell Regret:}Let $\beta$ denote a Blackwell optimal policy. Let $\gamma^{*}$ be as defined above as for Blackwell optimality, such that $V^{\beta}_{\gamma}(s) \geq V^{\pi}_{\gamma}(s)$, $\forall \gamma \in [\gamma^{*},1)$, $\forall \pi \in \Pi$, $\forall s \in \mathcal{S}$. Then for $\gamma \in [0,\gamma^{*})$,  we say $\gamma$ is myopic. Similarly, a policy, $\pi_{\gamma}$, is myopic if it is learned using a myopic $\gamma$. Similarly, we say for $\gamma \geq \gamma^{*}$ that $\gamma$ is \textit{Blackwell realizable}. For $\gamma_{learn} \in [0,1)$, we define Blackwell Regret, $\mathcal{R}_{\mathcal{B}}$. Let $\gamma' = max\{\gamma^{*}, \gamma_{learn} \}$. Then for a given policy $\pi_{\gamma_{learn}}$, \begin{align*}
\mathcal{R}_{\mathcal{B}}(\pi_{\gamma_{learn}}) = \mathbb{E}\big\{V^{\beta}_{\gamma'}(s) - V^{\pi_{\gamma_{learn}}}_{\gamma^{'}}(s) \big\},
\end{align*}
where the expectation is taken over initial state distribution. Hence, Blackwell regret is the regret accrued for using a given policy learned with $\gamma_{learn}$, when compared to a Blackwell optimal policy. Since it may be that $\gamma_{learn} > \gamma^{*}$, to ensure commensurability we require that $\gamma' = max\{\gamma^{*}, \gamma_{learn} \}$ in the definition, since under non-negative rewards, $\forall$ fixed $\pi$ $\gamma' < \gamma \in [0,1)$, $s \in \mathcal{S}$, we see that $V^{\pi}_{\gamma'}(s) < V^{\pi}_{\gamma}(s)$. It immediately follows that if $\gamma$ is myopic, and $\pi_{\gamma}^{*}$ is the optimal policy induced by $\gamma$, then $\mathcal{R}_{\mathcal{B}}(\pi_{\gamma}^{*}) > 0$. We see in the following lemma that Blackwell regret captures the very notion exemplified in the child-parent example previously given, in that for $\gamma < \gamma^{*}$, the Blackwell regret is simply the regret computed using $\gamma^{*}$. The regret of a given policy, with value evaluated at $\gamma$ is defined as 
\begin{align*}
\mathcal{R}(\pi; \gamma) = \mathbb{E}\big\{V^{\pi^{*}}_{\gamma}(s) - V^{\pi}_{\gamma}(s) \big\}.
\end{align*}

\begin{lemma}Let $\gamma \in [0,\gamma^{*})$, $\gamma^{*}$, as defined in Blackwell optimality. Then $\mathcal{R}_{\mathcal{B}}(\pi_{\gamma}) = \mathcal{R}(\pi; \gamma^{*})$.
\end{lemma}

\paragraph{}
Previous definitions of regret measure the difference in value of a given policy $\pi_{\gamma_{learn}}$ and the optimal value function, each evaluated with respect to a fixed $\gamma_{learn}$, since $V^{*}_{\gamma_{learn}}(s) = V^{\pi^{*}}_{\gamma_{learn}}(s)$, $\forall s \in \mathcal{S}$. As well, the $\gamma_{learn}$ used to learn the policy is typically also used to evaluate the value of that policy, and thus the regret. Blackwell regret differs in that it measures the difference in value of a given policy, $\pi_{\gamma_{learn}}$, and a Blackwell optimal policy, evaluated at $\gamma' = max$ $\{\gamma_{learn}, \gamma^{*}\}$ that favors a Blackwell optimal policy and value function.  In doing so, a policy that achieves zero Blackwell regret is either itself Blackwell optimal, or when considering a sufficiently long time horizon (as encoded by $\gamma'$), has the same value as a Blackwell optimal policy.

\section{Difficulty in Selecting Blackwell Realizable $\gamma$}

\paragraph{}
When implementing an RL algorithm that incorporates $\gamma$ discounting, typically no reasoning is provided to explain the choice of $\gamma$ used, though most often values of $\gamma$ are set around 0.9. $\gamma$ may be treated as a hyperparameter and a grid-search over values may be performed. However, even under these settings, the probability measure of non-myopic $\gamma$'s can become vanishingly small for various types of problems such as LHP's and sparse reward problems. Hence, any randomized $\gamma$ selection approach can have a vanishingly small probability of achieving non-zero Blackwell regret, as for the example in Figure 1 as $H$ grows. We show that selecting a non-myopic $\gamma$, that is, selecting a Blackwell realizable $\gamma$, is quite difficult without oracle knowledge of the problem. Moreover, even using a Blackwell realizable $\gamma$, an $\epsilon$-optimal policy may not even be gain optimal, let alone Blackwell optimal.

\paragraph{}
Ultimately we would like to consider MDP environments of a particular nature conducive to multi-task RL problems. The environments (problems) we are interested in are those such that for every task assigned to the agent, the optimal policy for that task induces a partition of the state space into non-empty subsets of transient and recurrent states, $\mathcal{S_{T}}, \mathcal{S_{A}}$. This is equivalent to saying that for each task, the optimal policy associated to the task induces a Markov chain on $\mathcal{S}$ which is unichain, or that the environment is multichain [13]. The intuition is that the environment is \textit{sufficiently controllable}, in the sense that the agent can direct the environment towards some preferable subset of the state space, and stay there indefinitely if needed, as encoded by the task MDP. For this paper we will consider a particular subset of such environments, where there are only two regions of the state space that produce non-zero rewards, and these two regions are maximally separated from one another. We demonstrate that even for such a simple class of MDPs, selecting Blackwell realizable discount factors can be arbitrarily hard.

\paragraph{}
More formally, we consider the class of MDPs with finite diameters. That is, $\exists$  $D < \infty$, such that
\begin{align*}
D = \underset{s \neq s' \in \mathcal{S}} {\mathrm{max}} \ \underset{\pi \in \Pi} {\mathrm{min}} \ \mathbb{E}_{\pi}\big\{\tau_{\pi}(s,s') \big\},
\end{align*}
 where $\tau_{\pi}(s,s')$ is the first hitting time of $s'$ when starting in state $s$, under $\pi$.  Hence, within the class of environments considered, it is possible to reach any state from any other starting state, and do so in a finite number of actions, in expectation, under \textit{some} policy. Furthermore, denote $s_{d} := s$ and $s_{H} := s'$ two states that realize the diameter $D$. Suppose $\exists$  $0 < r_{d} << R_{max} < \infty$, and $a, a' \in \mathcal{A}$ such that $r(s_{d}, a, s_{d}) = r_{d}$ and $r(s_{H}, a', s_{H}) = R_{max}$, and all other rewards are zero (e.g. $r \in \{0, r_{d}, R_{max}\}$).  Moreover, $p(s_{d} \vert a, s_{d}) = p(s_{H} \vert a', s_{H}) = 1$. Though this structure is quite specific, it abstractly represents two regions of the state space where actions exist that allow the agent to remain in those respective regions, and while remaining in that such region receive, on average, a positive rewards $r_{d}$ and $R_{max}$, respectively. An example of such an MDP can be seen in Figure 1. We call these particular environments \textit{distracting long horizon problems}, in the sense that due to the nature of the long horizon problem, the high reward region of the environment is many time steps away from an arbitrarily low reward region of the environment, with the rest of the environment producing no rewards. Given a state, such as $s_{0}$ in Figure 1, under a Blackwell optimal policy, the agent will not be distracted by the nearby, yet miniscule rewards, and will traverse to the high reward region, $s_{H}$. This setting is a slight step up in complexity from a simple goal based MDP where only a single state produces a positive reward.  We show that with oracle domain knowledge of features of the MDP (which we state below), one can select a Blackwell realizable $\gamma$ and solve for $\pi_{\gamma^{*}}^{*}$ in such distracting MDPs.  However, even with oracle knowledge of $\gamma^{*}$, $\forall \epsilon > 0$ selected, one may receive an $\epsilon$-optimal policy and value function that is not gain optimal, since for LHPs the value of a gain optimal policy and Blackwell optimal policy may differ by less than $\epsilon$. Interestingly, these results suggest a multi-step learning process for distracting MDP problems may be possible, which we leave for future work.

\paragraph{}
	We start with a proposition that shows that for this class of MDPs, being Blackwell optimal are exactly those policies that are \textit{not distracted}, in the sense that they are those policies that act solely to minimize the hitting time of the high reward state $s_{H}$.

\begin{proposition}Let $\mathcal{M}$ be a distracting long horizon MDP as described above. Then $\pi$ is Blackwell optimal $\iff$ $\pi \in \underset{\pi' \in \Pi} {\mathrm{arg min}}$ $\mathbb{E}_{\pi'}\big\{\tau_{\pi'}(s,s_{H}) \big\}$, $\forall s \in \mathcal{S}$.
\end{proposition}

We now provide results that show with oracle knowledge of $D, r_{d}, R_{max}$ we may select for $\gamma^{*}$ and thus for a Blackwell realizable discount factor. 

\begin{corollary}For any distracting long horizon MDP $\mathcal{M}$, as described above, if $D, r_{d}, R_{max}$ is known, then an RL algorithm can select $\gamma \geq \gamma^{*}$ and hence select a Blackwell realizable discount factor.
\end{corollary}

The following corollary shows that with oracle knowledge of only two of the following properties: $D, r_{d}$ \textit{and} $R_{max}$, then after committing to particular $\gamma \in [0,1)$ there exists a distracting long horizon MDP that is consistent with those MDP properties wherein $\pi^{*}_{\gamma}$ is not gain optimal, but $\forall  \gamma_{2} > \gamma$, $\pi^{*}_{\gamma_{2}}$ is Blackwell optimal.

\begin{corollary}Suppose for every distracting long horizon MDP $\mathcal{M}$, as described above, only two of $\{D, r_{d}, R_{max}\}$ is known. Let $\mathcal{K} \subset  \{D, r_{d}, R_{max}\}$, $\vert \mathcal{K} \vert = 2$, denote the MDP features known with oracle knowledge. Then $\forall \gamma \in [0,1)$ $\exists$ $\mathcal{M}$ consistent with $\mathcal{K}$, such that $\pi_{\gamma}^{*}$ is not gain optimal but $\forall \gamma' \in (\gamma,1)$, $\pi_{\gamma'}^{*}$ is Blackwell optimal.
\end{corollary}

These corollaries demonstrate that there exists sufficient domain knowledge for distracting long horizon MDPs to allow for the computation and use of a Blackwell realizable $\gamma$, however without complete domain knowledge of $\mathcal{M}$, any $\gamma$ selected may be myopic and may not even lead to a gain optimal policyl. These results suggest for distracting long horizon MDPs that a multi-step learning approach may be best, where in the first phase the agent learns the $D, r_{d}, R_{max}$, and then in the second phase, uses this knowledge to select for a non-myopic $\gamma$ to solve the task, however we leave such results for future work.

The next results show that even under with access to a Blackwell realizable $\gamma$, for distracting long horizon problems, then the value of a policy that is not gain optimal and that of a Blackwell optimal policy may be arbitrarily close (e.g. within $\epsilon$), hence any learning algorithm that returns a policy that is $\epsilon$-accurate to a Blackwell optimal policy may not even be gain optimal. Further, we provide empirical results that mirror our theoretical results.

\begin{corollary}Let $\epsilon > 0$. $\exists$ a distracting MDP, $\mathcal{M}$,  with Blackwell optimal $\gamma^{*} \in (0,1)$, and associated Blackwell optimal policy $\beta$, such that\\ $\vert\vert V^{\beta}_{\gamma^{*}}-V^{\pi_{\gamma^{*}}}_{\gamma^{*}}\vert\vert_{\infty} < \epsilon$, where $\pi_{\gamma^{*}}$ is not gain optimal.

\end{corollary}

\subsection{Policy Gaps and Pivot States}
\paragraph{}
Prior work has been done in putting forward measurements that can act as indicators of when learning an optimal policy may be difficult [10, 2] . [2] discuss the notion of an {\it action gap} at a given state $s$ that is the difference in expected value at that state between the optimal action and the second best action. More formally, let $A^{-}_{\pi}(s) = \mathcal{A} \setminus \{\pi^{*}(s)\}$. Then,
\begin{align*}
AG_{\pi^{*}}(s) = V^{\pi^{*}}_{\gamma}(s) - \underset{a \in A^{-}_{\pi}(s)} {\mathrm{max}}Q^{\pi^{*}}_{\gamma}(s,a).
\end{align*} 
[10] introduce the notion of the maximal action-gap (MAG) of a policy $\pi$  as
\begin{align*}
MAG(\mathcal{S}; \pi) = \underset{s \in \mathcal{S}} {\mathrm{max}}\big\{\underset{a \in \mathcal{A}} {\mathrm{max}} Q^{\pi}(s,a) - \underset{a \in \mathcal{A}} {\mathrm{min}} Q^{\pi}(s,a) \big\}.
\end{align*}
 Both studies argue that if their respective measurement is small, then learning the optimal policy can be hard, as it is hard to discern the value of the optimal action from one that is sub-optimal. While each may be useful, we argue that since the action gap measures the difference in value associated with abstaining {\it just once} from taking the optimal action, it doesn't truly measure the difference in value between two policies, nor the associated difficulty in discerning the value of one policy over another. The maximal action gap suffers from this as well. Moreover, [10] that under certain conditions the maximal action gap collapses to zero, making learning arbitrarily hard. However, in the Appendix section we prove that this condition only occurs in environments where the set of states that receive non-zero rewards must be transient $\forall \pi$.

We introduce a novel measurement, the policy gap, which is motivated by the action gap, discussed above. For $\mathcal{S}, \mathcal{A}, p, R$ fixed, policy $\pi$ and $s \in \mathcal{S}$, we define the policy gap, $PG_{\pi}(s)$,
\begin{align*}
PG_{\pi_{\gamma}}(s) =\underset{\pi_{\gamma}(s) \neq \pi'_{\gamma}(s)} {\underset{\pi'_{\gamma} \in \Pi_{\gamma}} {\mathrm{min}}} \bigg\{\vert V^{\pi_{\gamma}}_{\gamma}(s) -  V^{\pi'_{\gamma}}_{\gamma}(s) \vert \bigg\}.
\end{align*}

The policy gap at state $s$ is the smallest difference in value at that state between the query policy and any other policy that differs at $s$. Intuitively, if $\forall s$, $PG_{\pi_{\gamma}}(s)$ is large, then the ability to discern the optimal action and thereby learn a Blackwell optimal policy becomes easier. Conversely, if $\exists s \in \mathcal{S}$, sucht that  $PG_{\pi_{\beta}^{*}}(s) \to 0$ then at state $s$, called a \textit{pivot state}, the ability to discern the value of a Blackwell optimal policy, $\beta$, and another policy becomes increasingly hard. For an MDP where Blackwell optimal policies are non-trivial, that is not all policies are Blackwell optimal, and therefore $\gamma^{*} > 0$, then there exists such a pivot state. For the Theorem below, we use $\beta$ for a Blackwell optimal policy, and for any $\gamma$, we use $V^{\beta}_{\gamma}$ to represent the value function computed follow the Blackwell optimal policy and with discount factor $\gamma$.

\begin{theorem}(Pivot State Existence) Let $\beta$ be a non-trivial Blackwell optimal policy with $\gamma^{*} \in (0,1)$, where $\gamma^{*}$ as defined above such that $V^{\beta}_{\gamma}(s) \geq V^{\pi}_{\gamma}(s)$, $\forall \pi, s$, $\forall \gamma \in [\gamma^{*}, 1)$. If $\gamma < \gamma^{*} \implies \exists$ a pivot state $\tilde{s} \in \mathcal{S}$, $\exists$ $\tilde{\pi}_{\gamma} \in \Pi_{\gamma}$ where $\tilde{\pi}_{\gamma}(\tilde{s}) \neq \beta(\tilde{s})$, and
\begin{align*}
V^{\beta}_{\gamma}(\tilde{s}) < V^{\tilde{\pi}_{\gamma}}_{\gamma}(\tilde{s}) <V^{\tilde{\pi}_{\gamma}}_{\gamma^{*}}(\tilde{s}) \leq V^{\beta}_{\gamma^{*}}(\tilde{s}).
\end{align*}

Moreover,  $\underset{\gamma \to \gamma^{*}} {\mathrm{lim}} PG_{\beta_{\gamma}}(\tilde{s}) \to 0$.
\end{theorem}

\paragraph{}
Theorem 7 shows that for $\gamma$ values close to $\gamma^{*}$, there exists a pivot state such that the value of a Blackwell optimal policy at that state, when computed with $\gamma$, is arbitrarily close to the value of a different non-Blackwell optimal policy at the same state, when computed with $\gamma$. Intuitively, if the policy gap is arbitrarily close to zero, an RL algorithm is expected to have a greater difficulty evaluating the difference in value associated to such policies, and therefore have a greater difficulty in determining which is optimal. These results may suggest that without oracle knowledge of $\gamma^{*}$, an algorithm that attempts to search for $\gamma^{*}$ by increasing $\gamma$ iteratively would have increasing difficulty as $\gamma \to \gamma^{*}$.

\section{Experimental Results}
\paragraph{}
In this section we provide experimental results that further illustrate the phenomena discussed in previous sections. We investigate the difficulty of solving for Blackwell optimal policies in distracting long horizon MDPs, similar to those in Figure 1. For these experiments we use the MDP in Figure 2, with initial state $s_{d}$. We analytically solve for $\gamma^{*}$, and implement the delayed Q-learning PAC-MDP algorithm [16] for our experiments. We use $0.85 = \gamma_{2} > 0.84724541 =\gamma_{1} > \gamma^{*}$ in two sets of experiments, with $\gamma_{1} - \gamma^{*} < 10^{-10}$. For our experiments we use $\delta = 0.1$, and error tolerance $\epsilon_{1} = 0.05, \epsilon_{2} = 0.1$, where indices for $\epsilon, \gamma$ coincide for experiments. We run each set of experiments with a different random seed for 5 runs. The delayed Q-learning algorithms terminates when algorithm either finds itself in state $s_{d}$, greedily selecting $a_{1}$ and the \textit{Learn}(s,a) boolean flag is \textit{False}, or the algorithm finds itself in state $s_{H}$, greedily selecting $a_{2}$ and the \textit{Learn}(s,a) boolean flag is \textit{False}. Both situations indicate no further learning is possible, and the algorithm has converged on the $(\epsilon, \delta)$-optimal policy.

\paragraph{}
For the first set of experiments with $\gamma_{1} > \gamma^{*}$ and $\epsilon_{1} = 0.05$, the mean sample complexity required for convergence was $2.604679x10^{10} \pm 11219.2$. In each of the five experiments, the policy learned, $\hat{\pi}_{1}(s_{d}) = a_{1}, \hat{\pi}_{1}(s_{H}) = a_{2}$, which is not Blackwell optimal. Moreover, the algorithms terminates in state $s_{d}$, hence the policy is not even gain optimal. The policy gap at $s_{d}$ was also measured, and the mean policy gap, $\mu(PG(s_{d})) = 4.56x10^{-4} \pm 1.39x10^{-5}$. In the second set of experiments, using $\gamma_{2}, \epsilon_{2}$, the mean sample complexity across 5 runs was $1.00433x10^{11} \pm 7.223x10^{8}$. In each of the five experiments the policy learned was the Blackwell optimal policy, and the mean policy gap, $\mu(PG(s_{d})) = 2.5395x10^{-3} \pm 1.6865x10^{-5}$.

\paragraph{}
These results corroborate the theoretical results obtained. First, we see that for the experiments with $\gamma_{1}$, which is much closer to $\gamma^{*}$, we find that the policy gap at $s_{d}$ is much smaller than when compared to under $\gamma_{2} > \gamma_{1}$, as predicted by the theoretical results stated. More importantly, despite having oracle knowledge of $\gamma^{*}$ and selecting a Blackwell realizable $\gamma_{1}$, and implementing a PAC-MDP algorithm with commonly used values of $(\epsilon, \delta)$, no implementation returned a Blackwell optimal policy, and in fact did not even return a gain-optimal policy. These results further support the difficulty in arriving at Blackwell optimal policies for distracting long horizon MDPs. However, for $\gamma_{2}$ such that  $\gamma_{2} - \gamma^{*}  \approx 0.015275$, the Blackwell optimal policy was returned in all experiments. Though a positive result in some regards, it is also suggests that for distracting long horizon MDPs where $\gamma^{*} \to 1$, where the Lebesque measure $\lambda([\gamma^{*},1)) \to 0$, having the luxury of randomly selecting $\gamma \in [0,1)$ such that $gamma > \gamma^{*}$ becomes arbitrarily hard. Finally, these results corroborate the theoretical results showing the existence of a state where the policy gap approaches zero and that even with $\gamma^{*}$, with a commonly used $\epsilon$ error tolerance value, the $\epsilon$-optimal policy returned by a PAC-MDP algorithm was not even gain optimal.

\begin{figure}
\centering
\begin{tikzpicture}[scale=1.0, transform shape, node distance = 3.75cm]
        \tikzset{node style/.style={state, 
                                    fill=gray!20!white,
                                    circle}}
	 \node[node style]  (0) {$s_{d}$};
	\node[node style, right=of 0] (1) {$s_{H}$};

    \draw[>=latex,
          auto=left,
          every loop]
        
	(0) edge[loop above] node{$a_{1}: r_{d}$=0.1} (0)
	(0) edge[bend right, auto=left] node{$a_{2}: p = \frac{1}{500}$}(1)
	(0) edge[loop below] node{$a_{2}$: $p = \frac{499}{500}$}(0)
	(1) edge[bend right, auto=right] node{$a_{1}$}(0)
	(1)edge[loop above] node{$a_{2}: R_{max}$=1}(1);

\end{tikzpicture}
\caption{Distracting LHP. Actions are deterministic except for $a_{2}$ from state $s_{d}$. Only non-zero rewards are $r_{d}$ and $R_{max}$.} \label{fig:M2}
\end{figure}
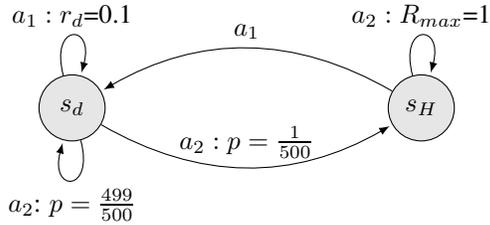

\section{Discussion and Related Work}
\paragraph{}

The topic of effects of $\gamma$ selection on policy quality has been of interest for several decades [17, 3, 13] with n-discount optimality and Blackwell optimality providing a global perspective on this relationship. These works recognize that for $\gamma$ discounting, an optimal policy may not be Blackwell optimal, and recognize that this problem is alleviated for $\gamma = 1$. However, there do not exist any known convergent algorithms with theoretical guarantees for the undiscounted setting. [3] also showed that for finite MDPs, as $\gamma \to 1$, the $\gamma$-discounted value function can be written as a Laurent series expansion, where each of the terms in this series is a scaled notion of optimality, with the first term being the gain, the second the bias, and so on. Using this construction, [13, 17] show there is both a sequence of nested equations for solving for the Laurent series coefficients, as well as a policy iteration method that is provably convergent for such a policy satisfying these equations for any finite term approximation of the Laurent series. More recently [11] utilized an exciting approach in function approximation by constructing value functions  using basis functions comprised of terms found within the Laurent series expansion.

\paragraph{}
[9] studies the relationship of $\gamma$ and reward functions with policy quality for goal based MDPs. They argue that $\forall \gamma<1$, an agent is not risk-averse, and prove that in the undiscounted setting and $r := -1$, $\forall (s,a,s')$, an agent is guaranteed to arrive at the goal state, however with $\gamma<1$ this is not so, as a shorter yet riskier path that may lead to non-goal absorbing state can have higher value than a longer, safer path to the goal. [12, 6, 7] are motivated by showing that using smaller $\gamma$ values may be advantageous. Besides having faster convergence rates, they argue smaller $\gamma$ values may also have better error. By decomposing the error or value difference between policies induced with different $\gamma$ values,  these decompositions have error terms dependant on the smaller $\gamma$ term, and another term that goes to zero as the two $\gamma$ values approach each other. These works argue that the best strategy is to find an intermediate $\gamma$ value that trades off the two terms. However, as is often the case in theoretical analysis of RL problems, the bounds are stated in terms of $V_{max}$, and for various values of $\gamma$, are vacuous as the bounds are higher than the absolute max error of $V_{max}$ (e.g. one policy only receiving zero rewards and another always receiving $R_{max}$). However, when the bounds are meaningful, without knowledge of the Blackwell optimal policy and associated value function, as shown in this study, even an $\epsilon$-optimal policy may not even be gain optimal.

\paragraph{}
For $\gamma \in [0,1)$ [10] define a hypothesis class, $\mathcal{H}_{\gamma}$ and show that 
\begin{align*}
\mathcal{H}_{\gamma} = \{\textbf{v} \in \mathbb{R}^{\mathcal{S}} : \vert \vert s \vert \vert_{\infty} \leq \frac{R_{max}}{1-\gamma}  \}.
\end{align*}
Under this framework, for a family of hypothesis classes $\mathcal{H}_{\gamma}$, indexed by $\gamma \in [0,1)$, we see that as $\gamma$ increases, $\{\mathcal{H}_{\gamma}\}_{\gamma}$ is a monotonically increasing sequence of hypothesis classes. [10, 6]  formalize that this also corresponds to an increase in measure of complexity, via the generalized Rademacher complexity,  $\mathscr{R}(\mathcal{H}_{\gamma})$, depends only on $V_{max}$. That is, 
\begin{align*}
\mathscr{R}(\mathcal{H}_{\gamma}) = \frac{R_{max}}{2(1-\gamma)}.
\end{align*}
As examined in [10], long horizon MDPs suffer in that $\gamma^{*}$ may be arbitrarily close to 1, which implies the complexity of realizable hypothesis classes for LHPs grows non-linearly with the horizon size (and $\gamma^{*}$). [6] argue that using $\gamma_{learn} < \gamma^{*}$ is therefore a mechanism akin to regularization, by selecting for a lower complexity hypothesis class one can prevent overfitting. Their results suggest using smaller $\gamma$ earlier in learning; however as discussed here, the quality of $\gamma$ being \textit{small} or \textit{large} is problem dependant, and without oracle knowledge of the problem is meaningless. Though [6] does not consider Blackwell optimality, an interesting result [Theorem 2] can easily be adapted here which shows that for $\gamma_{learn} < \gamma^{*}$ the loss as measured by $\vert \vert . \vert \vert_{\infty}$ for an approximately optimal policy $\hat{\pi}_{\gamma}$ using $\gamma$ and $n$ samples from each $(s,a) \in \mathcal{S}\times\mathcal{A}$ follows with probabiliy $> 1 - \delta$:

\begin{align*}
\vert \vert V^{\pi_{\gamma^{*}}^{*}} - \hat{V}^{\hat{\pi}_{\gamma}} \vert \vert_{\infty} &\leq \frac{\gamma^{*} - \gamma}{(1-\gamma^{*})(1-\gamma)}R_{max}\\
 &+ \frac{2R_{max}}{(1-\gamma)^{2}} \sqrt{\frac{1}{2n}log\big(\frac{2\vert\mathcal{S}\vert \vert\mathcal{A}\vert \vert \Pi_{\gamma} \vert}{\delta} \big)}.
\end{align*}

[6] argue that the tradoff between the two terms involves controlling the complexity of the policy class using a smaller $\gamma$, versus the error induced in the first time when using a smaller $\gamma$. Our results show that even as $n \to \infty$ and the second error term goes to zero, the first error term is fixed for any fixed $\gamma < \gamma^{*}$, and that without strong domain knowledge, even an $\epsilon$-optimal approximate policy may not even be gain optimal. [15] recently suggested $\gamma$-nets, a function approximation architecture that trains using a set of discount factors to learn value functions with respect to several timescales. The idea being that the approximation architecture can generalize and approximate the value of a state for any $\gamma$ if sufficiently trained.

\paragraph{}
The work presented here suggests that without considering Blackwell optimality and related concepts, theoretical bounds on value functions in RL may not provide meaningful and interpretable semantics with respect to the optimality of the resulting policy. An apt metaphor is that for a daredevil jumping across a canyon, coming $\epsilon$-close to being successful is arbitrarily bad. In that vein, our results show that for LHPs an $\epsilon$-Blackwell optimal policy may not even be gain optimal. In contrast, in the supervised learning setting, one may search over a particular hypothesis class and and arrive at some locally or globally optimal hypothesis $\hat{h}: \mathcal{X} \to \mathcal{Y}$, which obtains empirical accuracy of $p_{train}, p_{test}  \in [0,1]$ on the training and test datasets, respectively. Once a classifier is obtained, though one may not know that the Bayes optimal classifier risk may be, one does know that it can, at most, achieve 100\% accuracy, and hence in absolute terms, one can obtain meaning from the test and training accuracy of a classifier returned by some SL algorithm. However, the RL setting is not similar in these regards. Without oracle knowledge of the RL problem, the policy and value function returned by an RL algorithm, parameterized by $\gamma$, and any other parameters $\theta$, it is hard to say just how optimal such a policy, in fact, \textit{is}, thereby leaving a researcher in the same boat as the fictitious RL \textit{agent}: with results that are evaluative not instructive. 

\paragraph{}
Given that $\gamma$ discounting has such a strong effect on the induced hypothesis class, one may ask why discounting \textit{is even used}? Authors often cite concepts from utility theory such as inflation and interest to motivate the use of discounting. Such concepts for temporal valuation may be useful for agents, such as humans, with finite time horizons, however such intuitions may not necessarily be commensurable for infinite horizon agents. The use of $\gamma$ discounting in economic models is also of contention [18]. For economic and environmental policies, how should we discount the value of having a clean environment? Is discounting the future ethical in such settings? Might discounting the future lead us to an arbitrarily bad absorbing state? Utility theory has considered several qualities two utilitiy streams, $\{r_{t}\}_{t\geq 1}, \{r_{t}'\}_{t\geq1}$, may posses in forming binary relations used as orderings on value functions (utility streams) [8], including that of anonymity which essentially states that two utility streams are equal under an ordering if they are permutations of one another. Hence, anonymity can only be realized in the RL setting if $\gamma = 1$. These works introduce and argue for the use of Blackwell optimality in economics research. [14] answers the question why discounting is used: because it turns an infinite sum into a finite one. That is, it allows us to consider convergent series and therefore algorithms. It then follows that we are not selecting for $\pi^{*} \in \Pi$, but rather for $\pi^{*} \in \Pi \cap \{$policies that are representable by convergent algorithms$\}$. If RL algorithms are to be used and incorporated in real world processes and products, we raise the rhetorical question: \textit{What are the moral and ethical implications of purposefully running a sub-optimal infinite horizon algorithm, in perpetuity?}

\paragraph{}
The results provided in this paper suggest that iterative methods at arriving at $\gamma^{*}$ are problematic, suggesting a need for analytical methods of computing $\gamma^{*}$. However, even with $\gamma^{*}$, an approximately optimal policy may not even produce a gain optimal policy. For LHPs, as $\gamma^{*} \to 1$, since even using $\gamma > \gamma^{*}$ shares this unfortunate result, as demonstrated empirically in our experiments, what can be done to ensure solving for the Blackwell optimal policy? Recent advances in PAC-MDP algorithms [4] introduce PAC uniform learning, PAC algorithms that are $\epsilon$-optimal $\forall \epsilon$ simultaneously. Such algorithms must never explore then commit [4] , but rather must never stop learning, as it has been shown that such approaches are necessarily sub-optimal [5]. An interesting direction would be to consider the use of such algorithms for arriving at Blackwell optimal policies. 

\paragraph{}
Though Blackwell optimality is an ideal, for non-trivial LHPs it is possible that Blackwell optimal policies are hard to discern from policies that may not even be gain optimal. With such results being so dire, we suggest three main areas of focus for future research within the RL community. 1) Development of convergent algorithms for solving for n-discount optimal policies, with theoretical bounds, and efficient solution methods for arriving at the Laurent series expansion of a $\gamma$-discounted value function as $\gamma \to 1$; 2) Analytical solutions to $\gamma^{*}$; 3) Human preference and goal based RL.

\paragraph{}
Our main focus is on the third area of focus mentioned above. For any applied RL solution, for example a commercial product that relies on RL, we argue that ultimately the quality of a policy is judged by human preferences. Those implementing an RL solution method will receive a policy and a value function, and must evaluate if it is a sufficient solution to the given problem, or not. If not, the researcher will experiment with other parameters, including $\gamma$, and repeat until a policy is found that is sufficient. We call such an aproach based on human preference, and may be separate from the value function itself, and solely dependent on the behaviour of the policy. This can be seen by the works and discussions made recently [1] based on results on the CoastRunners domain. CoastRunners is a video game where the policy controls a boat in a racing game. The policy solved for by OpenAI resulted in the boat driving in circles, collecting rewards, rather than racing to the finish line and completing the race. Though OpenAI uses this as an example of a pathological behaviour induced by a faulty reward function, it can viewed as the induced behaviour by using a myopic $\gamma$ in a distracting LHP. OpenAI, and most others would agree, that the behaviour observed was \textit{pathological}, however, what makes it pathological? In fact, it was the optimal policy solved for, given the encoding of the MDP. We argue that what makes this pathological is simply that the policy \textit{didn't do what the researchers wanted it to do}, which was to win the race. For this reason, at this current state in RL research, we claim the ultimately, the quality of policies solved for are measured by their being deemed sufficient, as subjectively defined by the researcher. We claim that this is equivalent to the researcher ultimately desiring \textit{something} from the solved policy, and hence if this can be encoded as an indicator function, then goal based RL problems should be used, being some of the simplest classes of MDP problems.

\subsubsection*{Acknowledgments.}
The authors would like to thank Maia Fraser for discussions and thoughtful edits of prior versions of this manuscript.

\section{References}

[1] Openai
blog.
https://blog.openai.com/faulty-
reward-functions/

[2] Bellemare, M., Ostrovski, G., Guez, A., Thomas,
P., Munos, R.: Increasing the action gap: New op-
erators for reinforcement learning. In: AAAI. pp.
1476–1483 (2016)

[3] Blackwell, D.: Discrete dynammic programming.
Annals of Mathematical Stastics 33, 719–726
(1962)

[4] Dann, C., Lattimore, T., Brunskill, E.: Unifying pac
and regret: Uniform pac bounds for episodic re-
inforcement learning. In: Neural Information Pro-
cessing Systems (2016)

[5] Garivier, A., Kaufmann, E.: On explore-then-
commit strategies. In: Neural Information Process-
ing Systems (2017)

[6] Jiang, N., Kulesza, A., Singh, S., Lewis, R.: The
dependence of effective planning horizon on model
accuracy. In: AAMAS. vol. 14 (2015)

[7] Jiang, N., Singh, S., Tewari, A.: On structural prop-
erties of mdps that bound loss due to shallow plan-
ning. In: IJCAI (2016)

[8] Jonsson, A., Voorneveld, M.: The limit of dis-
counted utilitarianism. Theoretical Economics 13,
19–37 (2018)

[9] Koenig, S., Liu, Y.: The interaction of representa-
tions and planning objectives for decision-theoretic
planning tasks. Journal of Experimental and Theo-
retical Artificial Intelligence 14, 303–326 (2002)

[10] Lehnert, L., Laroche, R., van Seijen, H.: On value
function representation of long horizon problems.
In: 32nd AAAI Conference on Artificial Intelli-
gence. pp. 3457–3465 (2018)

[11] Mahadevan, S., Liu, B.: Basis construction from
power series expansions of value functions. In:
Lafferty, J.D., Williams, C.K.I., Shawe-Taylor, J.,
Zemel, R.S., Culotta, A. (eds.) Advances in Neu-
ral Information Processing Systems 23, pp. 1540–
1548. Curran Associates, Inc. (2010)

[12] Petrik, M., Scherrer, B.: Biasing approximate dy-
namic programming with a lower discount factor.
In: In Advances in Neural Information Processing
Systems. pp. 1265–1272 (2009)

[13] Puterman, M.: Markov decision processes: Dis-
crete stochastic dynammic programming. John Wi-
ley and sons, Inc. (1994)

[14] Schwartz, A.: A reinforcement learning method
for maximizing undiscounted rewards. In: ICML
(1993)

[15] Sherstan, C., MacGlashan, J., Pilarski, P.: Gener-
alizing value estimation over timescale. In: Pre-
diction and Generative Modeling in Reinforcement
Learning Workshop, FAIM (2018)

[16] Strehl, A., Li, L., Wiewiora, E., Langford, J.,
Littman, M.: Pac model-free reinforcement learn-
ing. In: ICML. pp. 881–888 (2006)

[17] Veinott, A.: Discrete dynammic programming with
sensitive discount optimality criteria. Annals of
Mathematical Stastics 40, 1635–166 (1969)

[18] Weitzman, M.: Gamma discounting. American
Economic Review 91, 260–271 (2001)


\section{Appendix}

\paragraph{A Comment on Bounds Related to the Maximal Action Gap:}
[10] define $\mathcal{S}_{C} \subseteq \mathcal{S}$ as a \textit{fully connected} subset of the state space. Despite the use of the term \textit{fully connected} which was intended to describe a subet of the state space that is reachable from anywhere within that subset, a more appropriate term is communicating, as fully connected has connotations that $\forall s, s' \in \mathcal{S}_{C}$, $\exists$ $a \in \mathcal{A}$ such that $p(s' \vert s, a)>0$. For this reason we will use the term communicating to describe $\mathcal{S}_{C}$. From this they define $V_{max,\gamma} = \underset{s \in \mathcal{S}_{C}} {\mathrm{max}}$ $V^{\pi_{\gamma}^{*}}(s)$. Lemma 2 of [10] states: MAG$(\mathcal{S}_{C}) \leq (1-\gamma^{D_{S_{C}}+1})V_{max,\gamma}$, where $D_{S_{C}}$ is the diameter of $\mathcal{S}_{C}$. From this, it is stated that if $V_{max,\gamma}$ is bounded as $\gamma \to 1$, then $\gamma \to 1$ implies that MAG $\to 0$. Though this implication is true, we show that $V_{max, \gamma}$ is bounded as $\gamma \to 1$ if and only if under all policies the expected number of times a non-zero reward is obtained under $\pi_{\gamma}^{*}$ is finite. This means that under $\textit{all}$ policies, $\textit{all}$ non-zero rewards are transient. Hence, such a result applies to a rather vacuous subset of MDPs.

\begin{proposition}Let $\mathcal{M} = \langle \mathcal{S}, \mathcal{A}, p, R, \gamma \rangle$ be an MDP such that $\vert \mathcal{S} \vert < \infty$ and $\vert \mathcal{A} \vert < \infty$. Then $\exists$ $M < \infty \ni \lim_{\gamma \to 1} V_{max,\gamma} \leq M \iff \forall \pi$ $\mathbb{E}_{\pi_{\gamma}}\big\{\sum_{t=1}^{\infty} \mathbbm{1}r_{t} \neq 0  \big\} < \infty$.
\end{proposition}

\begin{proof}Suppose $\exists$ $M < \infty$ such that $V_{max,\gamma} \leq M$ as $\lim_{\gamma \to 1}$. WLOG, since $R \in [0, R_{max}]$, we may assume $R_{max} > 0$, since otherwise this statement is trivial. Clearly, $\mathcal{S} \neq \mathcal{S}_{C}$, since otherwise, as there exists at least one transition that induces a non-zero reward, $r > 0$, then at worst a policy may traverse the entire diameter of $\mathcal{S}_{C}$ to receive a reward $r$ and do so for perpetuity. That is,\\

\begin{align*}
V_{max,\gamma} &= \underset{s \in \mathcal{S}_{C}} {\mathrm{max}}\ V^{\pi_{\gamma}^{*}}(s)\\
&\geq r\gamma^{D_{S_{C}}} + r\gamma^{2D_{S_{C}}} + ...\\
&= \sum_{t=1}^{\infty}r\gamma^{tD_{S_{C}}}\\
&= r\gamma^{D_{S_{C}}}\sum_{t=0}^{\infty}\gamma^{tD_{S_{C}}}\\
&= \frac{r\gamma^{D_{S_{C}}}}{1-\gamma^{D_{S_{C}}}}
\end{align*}

But clearly,

\begin{align*}
M \geq \frac{r\gamma^{D_{S_{C}}}}{1-\gamma^{D_{S_{C}}}} &\iff \frac{M}{r} -\frac{M\gamma^{D_{S_{C}}}}{r} \geq \gamma^{D_{S_{C}}}\\
&\iff \frac{M}{r} \geq \gamma^{D_{S_{C}}}(1+\frac{M}{r})\\
&\iff \frac{M}{r+M} \geq \gamma^{D_{S_{C}}}\\
&\iff \big(\frac{M}{r+M}\big)^{\frac{1}{D_{S_{C}}}} \geq \gamma
\end{align*}

So for $\gamma >\big(\frac{M}{r+M}\big)^{\frac{1}{D_{S_{C}}}}$ we have $M < V_{max,\gamma}$. This shows that $\mathcal{S} \neq \mathcal{S}_{C}$. Hence, for $\mathcal{S}_{C} \subsetneq \mathcal{S}$, it must be that $\mathcal{S}_{C}$ is transient under $\pi_{\gamma}^{*}$, since otherwise if $\exists$ $T \in \mathbb{N}$ such that $\forall t \geq T$, $s_{t} \in \mathcal{S}_{C}$, then again by the same argument above, $\exists$ $\gamma \in [0,1)$ such that $\forall \gamma' \geq \gamma$, $V_{max,\gamma'} > M$. Hence $\mathcal{S}_{C}$ must be transient under $\pi_{\gamma}^{*}$. Now, since $\vert \mathcal{S} \vert < \infty$, then $\vert \mathcal{S} \setminus \mathcal{S}_{C}\vert < \infty$. Hence for $\pi_{\gamma}^{*}$, $\exists$ $S_{A} \subseteq  \mathcal{S} \setminus \mathcal{S}_{C}$ such that $S_{A}$ is irreducible and positive recurrent (e.g. absorbing). We claim that there must not be any possible non-zero rewards within $S_{A}$. Let $T = \underset{s \in \mathcal{S}} {\mathrm{max}} \mathbb{E}_{\pi_{\gamma}^{*}}\big\{\tau(s,S_{A}) \big\}$ be the maximum expected first hitting time of reaching the absorbing subset of the state space $S_{A}$ under $\pi_{\gamma}^{*}$. By a similar argument as above, there cannot be any positive rewards in $S_{A}$, since otherwise $\exists$ $s_{A} \in S_{A}$ such that $V^{\pi_{\gamma}^{*}}(s_{A}) \to \infty$ as $\gamma \to 1$. If this is true, then $\exists$ $s' \in \mathcal{S}_{C}$ such that $V_{max, \gamma} = \underset{s \in \mathcal{S}_{C}} {\mathrm{max}}\ V^{\pi_{\gamma}^{*}}(s) \geq \gamma^{T}V^{\pi_{\gamma}^{*}}(s_{A})$, and therefore $V_{max,\gamma} \to \infty$ as $\gamma \to 1$.

Hence, as $\gamma \to 1$, $\pi_{\gamma}^{*}$ obtains non-zero rewards for only a finite number of time steps. Due to the optimality of $\pi_{\gamma}^{*}$, then this must be true for any policy $\pi_{\gamma}$. Hence it must be that all rewards in $\mathcal{M}$ are transient.\\

For the reverse implication, suppose that $\forall \pi$ $\mathbb{E}_{\pi_{\gamma}}\big\{\sum_{t=1}^{\infty} \mathbbm{1}r_{t} \neq 0  \big\} < \infty$. Let $T$ be defined as above, as the maximum expected hitting time of the absorbing subset $S_{A}$ which contains no non-zero rewards. $S_{A}$ must exist, by a similar argument as above. Then we have, $\forall s \in \mathcal{S}$
\begin{align*}
V^{\pi_{\gamma}^{*}}(s) &\leq \sum_{t=1}^{T}R_{max}\gamma^{t-1}\\
&\leq TR_{max} < \infty.
\end{align*}
Hence $V_{max,\gamma}$ is bounded as $\gamma \to 1$.
\end{proof}

\paragraph{}The maximum action gap bounds collapse to zero for an infinite horizon problem, as $\gamma \to 1$, but only for environments where $\textbf{all}$ the rewards are transient. [10] argue that representing the value function for such class of MDPs is quite difficult as $\gamma \to 1$, however such a class of environments are best solved using episodic MDP approaches, with $\gamma = 1$. Since for any policy the number of time steps where a positive reward is possible is finite, then finding an optimal policy is only relevant for the first $T < \infty$ time steps, since afterwards the behaviour becomes irrelevant. [13] shows such domains can be converted to undiscounted episodic tasks. In doing so, the hypothesis space is completely different, as only value functions $V \in [0, R_{max}T]^{\mathcal{S}}$ need be considered, which have no dependancy on $\gamma$, hence the Radamacher complexity results stated previously do not apply here. A multi-step learning approach of first learning $T$, then applying an episodic RL algorithmic approach is ideal for such environments.

\paragraph{Proof of Lemma 1}
\begin{proof}
Let $\beta$ be a Blackwell optimal policy with associated $\gamma^{*}$. Note that $\gamma^{'} = \gamma^{*}$ follows from the hypothesis and definition of $\gamma^{'}$ for Blackwell regret. Then,
\begin{align*}
R_{\mathcal{B}}(\pi_{\gamma}) &= \mathbb{E}\big\{V^{\beta}_{\gamma^{*}}(s) - V^{\pi_{\gamma}}_{\gamma^{*}}(s) \big\}\\
&= \mathbb{E}\big\{V^{\beta}_{\gamma^{*}}(s) - V^{*}_{\gamma^{*}}(s) + V^{*}_{\gamma^{*}}(s) - V^{\pi_{\gamma}}_{\gamma^{*}}(s) \big\}\\
&= \mathbb{E}\big\{V^{\beta}_{\gamma^{*}}(s) - V^{*}_{\gamma^{*}}(s)\big\} + \mathbb{E}\big\{V^{*}_{\gamma^{*}}(s) - V^{\pi_{\gamma}}_{\gamma^{*}}(s) \big\}\\
&= 0 + \mathbb{E}\big\{V^{*}_{\gamma^{*}}(s) - V^{\pi_{\gamma}}_{\gamma^{*}}(s) \big\}\\
&= \mathcal{R}(\pi; \gamma^{*})
\end{align*}
\end{proof}

\paragraph{Proof of Proposition 2}
\begin{proof} Let $\pi$ be a Blackwell optimal policy, then it is bias optimal which clearly must minimize the expected hitting time of $s_{H}$. For the reverse implication, let $\pi$ be the policy that minimizes the expected hitting time of $s_{H}$. Let $\gamma = \sqrt[^D]{\frac{r_{d}}{R_{max}}}$, with $D, r_{d}, R_{max}$ defined in the text. Then it follows that $\gamma^{*} = \gamma$, since otherwise $\forall \gamma' < \gamma, \exists \pi'$ such that $V^{\pi'}_{\gamma'}(s_{d}) > V^{\pi}_{\gamma'}(s_{d})$. It clearly follows that under any policy $\mu$, $\forall \gamma \geq \gamma^{*}$, $V^{\pi}_{\gamma} \geq V^{\mu}_{\gamma}$, and therefore $\pi$ is a Blackwell optimal policy.
\end{proof}

\paragraph{Proof of Corollary 3}
\begin{proof}Let $\mathcal{M}$ be a distracting MDP as described above, with $D, r_{d}, R_{max}$ known to the algorithm. Let $\pi^{*}$ be the Blackwell optimal policy learned and evaluated with $\gamma^{*}$. By the previous Proposition, then $\pi^{*}$ is the policy that takes the shortest path from any state to $s_{H}$, and as given in the proof of said Proposition, $\gamma^{*} = \sqrt[^D]{\frac{r_{d}}{R_{max}}}$. Moreover, from Proposition 2 $\forall \gamma < \gamma^{*}$, $\exists$ $\pi \neq \pi^{*} \ni V^{\pi_{\gamma}}_{\gamma}(s_{d}) < V^{\pi_{\gamma^{*}}^{*}}_{\gamma^{*}}(s_{d}).$ This follows for any policy that does not minimize the expected first hitting time of $s_{H}$. Hence $\forall \gamma \geq \gamma^{*} = \sqrt[^D]{\frac{r_{d}}{R_{max}}}$. Hence, with knowledge of $D, r_{d}, R_{max}$, $\gamma^{*}$ can be computed and therefore a realizable discount factor may be selected.
\end{proof}

\paragraph{Proof of Corollary 4}
\begin{proof}
First, given $r_{d}, R_{max}$, and let $\gamma \in [0,1)$. For $s_{d}, s_{H}$ as defined above, it suffices to show as in the previous proposition $\exists D>0 \ni$, $\forall \gamma' > \gamma$, for the induced optimal policies $\pi_{\gamma}^{*}, \pi_{\gamma'}^{*}$, and the Blackwell optimal policy $\beta$,
\begin{align*}
V^{\pi_{\gamma}^{*}}_{\gamma}(s_{d}) = \frac{r_{d}}{1-\gamma} &> \frac{R_{max}\gamma^{D}}{1-\gamma} = V^{\beta_{\gamma}}_{\gamma}(s_{d})\ but\\
V^{\pi_{\gamma'}^{*}}_{\gamma'}(s_{d}) = \frac{r_{d}}{1-\gamma'} &\leq \frac{R_{max}\gamma'^{D}}{1-\gamma'} = V^{\beta_{\gamma'}}_{\gamma'}(s_{d})
\end{align*}

Hence, it suffices to show $\exists D, \gamma' \ni$:
\begin{align*}
V^{\pi_{\gamma}^{*}}_{\gamma}(s_{d}) &= \frac{r_{d}}{1-\gamma}\\
&< \frac{R_{max}\gamma^{D}}{1-\gamma} = V^{\beta_{\gamma}}_{\gamma}(s_{d})\\
 &\leq \frac{R_{max}\gamma'^{D}}{1-\gamma'} = V^{\beta_{\gamma'}}_{\gamma'}(s_{d})
\end{align*}

Let $D$ = sup $\{D' \vert D' <\ log(\frac{r_{d}}{R_{max}}) - log(\gamma)\}$, and set $\gamma' := \sqrt[D]{\frac{r_{d}}{R_{max}}}$. Then $D, \gamma'$ satisfy the claim, and with initial state distribution being a point mass at $s_{d}$, we have $\pi_{\gamma}^{*}$ is not gain optimal, as $\rho^{\pi_{\gamma}^{*}} = r_{d}$, but $\forall \tilde{\gamma} > \gamma$, it follows that $\pi_{\tilde{\gamma}}^{*}$ is Blackwell optimal.

Without loss of generality, the same proof technique can be applied when either $r_{d}, D$ are known, and $\gamma \in [0,1)$ is fixed, as well as if $R_{max}, D$ are known, and $\gamma \in [0,1)$ is fixed.
\end{proof}

\paragraph{Proof of Corollary 5}
\begin{proof}This follows as a Corollary from Theorem 6, and Proposition 2, since $\exists$ a pivot state $\tilde{s}$ where the policy gap vanishes. It is easy to see that under Proposition 2 and the previous two Corollaries that followed, $s_{d}$ is a pivot state. Let $\tilde{\pi}$ equal the Blackwell optimal policy, $\beta$, at every state except, $\tilde{\pi}(s_{d}) = a_{stay}$, noting that $r(s_{d}, a_{stay}, s_{d}) = r_{d}$. Then $\tilde{\pi}$ is not gain optimal as $\rho^{\tilde{\pi}} = r_{d} < R_{max} = \rho^{\beta}$, yet $\forall s \in \mathcal{S}\setminus \{s_{d}\}$, $V^{\beta}_{\gamma^{*}}(s) = V^{\tilde{\pi}_{\gamma^{*}}}_{\gamma^{*}}(s)$, and for $s_{d}$, we see that $V^{\tilde{\pi}}_{\gamma^{*}}(s_{d}) = \frac{r_{d}}{1-\gamma^{*}}$, while $V^{\beta}_{\gamma^{*}}(s_{d}) =  \frac{\gamma^{*^{D}}R_{max}}{1-\gamma^{*}}$. Then $r_{d}, D, R_{max}$ can be set such that $\forall \epsilon > 0$, $\vert \vert V^{\beta}_{\gamma^{*}} - V^{\tilde{\pi}_{\gamma^{*}}}_{\gamma^{*}} \vert \vert_{\infty} < \epsilon$.
\end{proof}

\paragraph{Proof of Theorem 6}
\begin{proof}
Let $\gamma < \gamma^{*}$. By definition of Blackwell optimality, then $\exists \tilde{s} \in \mathcal{S}$, $\exists \tilde{\pi}_{\gamma}$ such that $V^{\tilde{\pi}_{\gamma}}_{\gamma}(\tilde{s}) > V^{\beta}_{\gamma}(\tilde{s})$. Moreover, since all rewards are non-negative, $\forall \pi$, $s \in \mathcal{S}$, $\forall \gamma_{1} < \gamma_{2}$ it follows that $V^{\pi}_{\gamma_{1}}(s) < V^{\pi}_{\gamma_{2}}(s)$. That is, increasing $\gamma$ while keeping the policy constant can only increase the magnitude of the value function. Hence we have as well,
\begin{align*}
&V^{\tilde{\pi}_{\gamma}}_{\gamma}(\tilde{s}) < V^{\tilde{\pi}_{\gamma}}_{\gamma^{*}}(\tilde{s}),\ and\\
&V^{\beta}_{\gamma}(\tilde{s}) < V^{\beta}_{\gamma^{*}}(\tilde{s}).
\end{align*}
Together, we see that
\begin{align*}
V^{\beta}_{\gamma}(\tilde{s}) < V^{\tilde{\pi}_{\gamma}}_{\gamma}(\tilde{s}) <V^{\tilde{\pi}_{\gamma}}_{\gamma^{*}}(\tilde{s}) \leq V^{\beta}_{\gamma^{*}}(\tilde{s}).
\end{align*}
It remains to show that $\tilde{\pi}(\tilde{s}) \neq \beta(\tilde{s})$ and $\underset{\gamma \to \gamma^{*}} {\mathrm{lim}} PG_{\beta_{\gamma}}(\tilde{s}) \to 0$. We may assume the former, since if, infact  $\tilde{\pi}(\tilde{s}) = \beta(\tilde{s})$, then
\begin{align*}
 V^{\beta}_{\gamma}(\tilde{s}) =\ &\mathbb{E}\{r(\tilde{s}, \beta(\tilde{s})) + \gamma V^{\beta}_{\gamma}(s')\}\\
<\ &\mathbb{E}\{r(\tilde{s}, \tilde{\pi}(\tilde{s})) + \gamma V^{\tilde{\pi}_{\gamma}}_{\gamma}(s') \} = V^{\tilde{\pi}_{\gamma}}_{\gamma}(\tilde{s})\iff\\
&\mathbb{E}\{r(\tilde{s}, \beta(\tilde{s})) + \gamma V^{\beta}_{\gamma}(s')\}\\
<\ &\mathbb{E}\{r(\tilde{s}, \beta(\tilde{s})) + \gamma V^{\tilde{\pi}_{\gamma}}_{\gamma}(s') \}\iff\\
&\mathbb{E}\{\gamma V^{\beta}_{\gamma}(s')\} < \mathbb{E}\{\gamma V^{\tilde{\pi}_{\gamma}}_{\gamma}(s') \}\iff \\
&\mathbb{E}\{V^{\beta}(s')\} < \mathbb{E}\{V^{\tilde{\pi}_{\gamma}}(s') \}
\end{align*}
Since the expectation is taken over MDP dynamics, and both policies selected the same action at $\tilde{s}$, then distribution over successor states are the same. If there are no successor states, $s'$, where $\pi^{*}(s') \neq \tilde{\pi}(s')$ then this inequality continues to the successors of the successor states. However, this process cannot continue indefinitely, since otherwise the two Markov chains induced by $\beta$ and $\tilde{\pi}$ beginning at $\tilde{s}$ are therefore coupled, and with the same dyamics and $\gamma$, must have the same value. Therefore the two policies must differ at atleast one state where the preceeding value function inequality is true. For this reason, WLOG, we assume this state is $\tilde{s}$.

Finally, to show $\underset{\gamma \to \gamma^{*}} {\mathrm{lim}} PG_{\beta_{\gamma}}(\tilde{s}) \to 0$.  This directly follows, as $\forall \gamma < \gamma^{*}$ we have 
\begin{align*}
&V^{\beta}_{\gamma}(\tilde{s}) < V^{\tilde{\pi}_{\gamma}}_{\gamma}(\tilde{s})  <V^{\tilde{\pi}_{\gamma}}_{\gamma^{*}}(\tilde{s}) \leq V^{\beta}_{\gamma^{*}}(\tilde{s})\\
\to &0 < V^{\tilde{\pi}_{\gamma}}_{\gamma}(\tilde{s}) - V^{\beta}_{\gamma}(\tilde{s})< V^{\beta}_{\gamma^{*}}(\tilde{s})-V^{\beta}_{\gamma}(\tilde{s})\\
\to &0 < PG_{\beta_{\gamma}}(\tilde{s}) < V^{\tilde{\pi}_{\gamma}}_{\gamma}(\tilde{s}) - V^{\beta}_{\gamma}(\tilde{s})< V^{\beta}_{\gamma^{*}}(\tilde{s})-V^{\beta}_{\gamma}(\tilde{s})\\
\to &0< \underset{\gamma \to \gamma^{*}} {\mathrm{lim}} PG_{\beta_{\gamma}}(\tilde{s}) < \underset{\gamma \to \gamma^{*}} {\mathrm{lim}}  V^{\beta}_{\gamma^{*}}(\tilde{s})-V^{\beta}_{\gamma}(\tilde{s})\\
\to &0<  \underset{\gamma \to \gamma^{*}} {\mathrm{lim}} PG_{\beta_{\gamma}}(\tilde{s}) < \underset{\gamma \to \gamma^{*}} {\mathrm{lim}} \mathbb{E}_{\beta}\big\{\sum_{t=1}^{\infty} \gamma^{*^{t-1}}r_{t} - \gamma^{t-1}r_{t}\big\}\\
\to &0<  \underset{\gamma \to \gamma^{*}} {\mathrm{lim}} PG_{\beta_{\gamma}}(\tilde{s}) < \underset{\gamma \to \gamma^{*}} {\mathrm{lim}} \mathbb{E}_{\beta}\big\{\sum_{t=1}^{\infty} (\gamma^{*^{t-1}} -\gamma^{t-1})r_{t}\big\}
\end{align*}
Since $\underset{\gamma \to \gamma^{*}} {\mathrm{lim}} \mathbb{E}_{\beta}\big\{\sum_{t=1}^{\infty} (\gamma^{*^{t-1}} -\gamma^{t-1})r_{t}\big\} \to 0$, it follows that $\underset{\gamma \to \gamma^{*}} {\mathrm{lim}} PG_{\beta_{\gamma}}(\tilde{s}) \to 0$.
\end{proof}

\end{document}